\documentclass[letterpaper]{article} % DO NOT CHANGE THIS
\usepackage{hyperref}
\usepackage{aaai24}  % DO NOT CHANGE THIS
\usepackage{times}  % DO NOT CHANGE THIS
\usepackage{helvet}  % DO NOT CHANGE THIS
\usepackage{courier}  % DO NOT CHANGE THIS
\usepackage{graphicx} % DO NOT CHANGE THIS
\usepackage{natbib}  % DO NOT CHANGE THIS AND DO NOT ADD ANY OPTIONS TO IT
\usepackage{caption} % DO NOT CHANGE THIS AND DO NOT ADD ANY OPTIONS TO IT
\usepackage{algorithm}
\usepackage{amsmath, amsthm, amsfonts, mathtools}
\usepackage{booktabs}
\usepackage{nicematrix}
\usepackage{algpseudocodex}
\usepackage{circledsteps}
\usepackage{subfig}
\usepackage{enumitem}
\usepackage[capitalize]{cleveref}
\usepackage{etoolbox}
\usepackage{dsfont}
\usepackage{xfrac}
\usepackage{stmaryrd}
\usepackage{fancyhdr}

\hypersetup{colorlinks,linkcolor={red},citecolor={blue},urlcolor={magenta}}

\DeclareMathOperator*{\argmax}{arg\,max}

\newcommand\numberthis{\addtocounter{equation}{1}\tag{\theequation}}

\newenvironment{talign}
 {\align}
 {\endalign}
\newenvironment{talign*}
 {\csname align*\endcsname}
 {\endalign}

\newtheorem{theorem}{Theorem}
\newtheorem{lemma}{Lemma}

\newtheorem{lem}{Lemma}[]
\Crefname{lem}{Lemma}{Lemmas}
\Crefname{lemenumi}{Lemma}{Lemmas}
\AtBeginEnvironment{lem}{%
    \crefalias{enumi}{lemenumi}%
    \setlist[enumerate,1]{
        label={\textit{(\arabic*)}},
        ref={\thelem.\arabic*}
    }%
}

\usepackage{newfloat}
\usepackage{listings}
\DeclareCaptionStyle{ruled}{labelfont=normalfont,labelsep=colon,strut=off} % DO NOT CHANGE THIS
\floatstyle{ruled}
\newfloat{listing}{tb}{lst}{}
\floatname{listing}{Listing}
\title{Sample Efficient Reinforcement Learning with Partial Dynamics Knowledge}
\author{
    Meshal Alharbi,
    Mardavij Roozbehani,
    Munther Dahleh
}
\affiliations{
    Laboratory for Information and Decision Systems, Massachusetts Institute of Technology\\
    Cambridge, MA 02139, USA\\
    \{meshal, mardavij, dahleh\}@mit.edu
}

\begin{document}

\maketitle

\begin{abstract}
The problem of sample complexity of online reinforcement learning is often studied in the literature without taking into account any partial knowledge about the system dynamics that could potentially accelerate the learning process. In this paper, we study the sample complexity of online Q-learning methods when some prior knowledge about the dynamics is available or can be learned efficiently. We focus on systems that evolve according to an additive disturbance model of the form $S_{h+1} = f(S_h, A_h) + W_h$, where $f$ represents the underlying system dynamics, and $W_h$ are unknown disturbances independent of states and actions. In the setting of finite episodic Markov decision processes with $S$ states, $A$ actions, and episode length $H$, we present an optimistic Q-learning algorithm that achieves $\tilde{\mathcal{O}}(\textsc{Poly}(H)\sqrt{T})$ regret under perfect knowledge of $f$, where $T$ is the total number of interactions with the system. This is in contrast to the typical $\tilde{\mathcal{O}}(\textsc{Poly}(H)\sqrt{SAT})$ regret for existing Q-learning methods. Further, if only a noisy estimate $\hat{f}$ of $f$ is available, our method can learn an approximately optimal policy in a number of samples that is independent of the cardinalities of state and action spaces. The sub-optimality gap depends on the approximation error $\hat{f}-f$, as well as the Lipschitz constant of the corresponding optimal value function. Our approach does not require modeling of the transition probabilities and enjoys the same memory complexity as model-free methods.
\end{abstract}

\begin{table*}[t]
\centering
\small
\begin{NiceTabular}{c|l|lll}
    \toprule
    \textbf{Structural Knowledge}   & \textbf{Algorithm} & \textbf{Regret} & \textbf{Time} & \textbf{Space} \\
    \midrule
    \Block{4-1}{Without}            & UCBVI \cite{azar2017minimax}                      & $\tilde{\mathcal{O}}(\sqrt{H^2SAT}+H^4S^2A)$      & $\mathcal{O}(S^2AT)$ & $\mathcal{O}(S^2AH)$ \\
                                    & UBEV \cite{dann2017unifying}                      & $\tilde{\mathcal{O}}(\sqrt{H^3SAT}+H^2S^3A^2)$    & $\mathcal{O}(S^2AT)$ & $\mathcal{O}(S^2AH)$ \\
                                    & UCB-H \cite{jin2018q}                             & $\tilde{\mathcal{O}}(\sqrt{H^4SAT})$              & $\mathcal{O}(T)$     & $\mathcal{O}(SAH)$ \\
                                    & Q-EarlySettled-Advantage \cite{li2021breaking}    & $\tilde{\mathcal{O}}(\sqrt{H^2SAT} + H^6SA)$      & $\mathcal{O}(T)$     & $\mathcal{O}(SAH)$ \\
    \midrule
    \small With                     & UCB-f (\textbf{this work})                        & $\tilde{\mathcal{O}}(\sqrt{H^6T})$                & $\mathcal{O}(SAT)$   & $\mathcal{O}(SAH)$ \\
    \bottomrule
\end{NiceTabular}
\caption{Summary of regret, time, and space complexities of Q-learning algorithms on finite episodic MDPs.}
\label{table:complexity}
\end{table*}

\section{Introduction}

\subsubsection{Motivations.} Reinforcement learning (RL) algorithms have achieved tremendous success in various domains, and the literature has seen an influx of methods with provable non-asymptotic guarantees. Many of these results are problem-agnostic and the theoretical guarantees are limited by the worst-case scenarios. At the same time, one would hope that the sample efficiency of RL could be tied to the inherent measures of complexity of the problem's dynamics, with ways to accelerate learning if some prior knowledge about the problem is available. Indeed, the following question in its full generality is still open:
\begin{center}
\textit{
How can we improve the sample efficiency of RL methods when some structure about the underlying dynamics is known or can be learned efficiently?
}
\end{center}
Practical methods that aim to address the previous question should carefully balance the sample complexity (number of interactions with the system) and computations (time and memory) spent in planning or improving policies. Our primary objective is to incorporate prior structural information without extensive offline computation or access to strong computational oracles.

\subsubsection{Systems of Interest.} 
Learning complexity in RL is multifaceted, and is influenced by uncertainty in both rewards and dynamics (see, e.g., \citet{lu2023reinforcement} for an introductory treatment on the subject). In the interest of tractability, herein, we focus our attention on a specific class of RL problems with uncertainty in the dynamics only. We consider a class $\mathcal{M}$ of Markov Decision Processes (MDPs) where the state transition is governed by:
\begin{equation}
    S_{h+1} = f(S_h, A_h) + W_h
\label{eq:system}
\end{equation}
where $f$, which may be partially or approximately known, represents the dynamics or fixed structure of the system, and $W_h$ are unknown disturbances independent of states and actions. Further, we assume the reward $r$ function is known; and thus, all the uncertainty is in the state transitions. Partial dynamics knowledge is modeled via the availability of an approximation $\hat{f}$  of $f$ that satisfies $\Vert \hat{f}-f\Vert_\infty \leq \zeta/2$. Note that this is different than \emph{sim-to-real} literature, where it is often assumed that one has a detailed probabilistic description of an approximate model \cite{jiang2018pac,feng2019does} or an approximate generative model where we can sample an arbitrary amount of trajectories without cost \cite{chen2021understanding}. See the next section for a detailed comparison. Furthermore, even if $f$ is fully known, one cannot apply dynamic programming here as the disturbances $W_h$ are not known even in distribution (in fact, this can be the main challenge in many practical problems). Finally, we shall see that restricting to linear disturbance models is beneficial for devising computationally efficient algorithms.

Many applications fit the setting described above. In operational tasks, like inventory control and demand response problems \cite{vazquez2019reinforcement}, the structural function $f$ is often known and the challenge lies in optimizing for the unknown non-stationary demand signals. In control tasks, where the dynamics rise from parametric physical models, one might only know $f$ approximately due to parametric uncertainty \cite{bhattacharyya2017robust, taylor2021towards}. For all of these problems, the assumption of known reward is not restrictive as the reward function is a design parameter.

\subsubsection{Contributions.} We study the previously stated research question under the finite episodic MDPs setting. Our main findings can be summarized as follows:
\begin{itemize}
    \item Under the perfect knowledge of $f$, we propose an optimistic Q-learning method that achieves $\tilde{\mathcal{O}}(\sqrt{H^6T})$\footnote{The $\tilde{\mathcal{O}}(\cdot)$ notation is analogous to the standard $\mathcal{O} (\cdot)$ notation while ignoring $\textsc{PolyLog}(S,A,H,K,1/p)$ dependencies.} regret with high probability without requiring any special assumptions on $f$ or the distribution of the unknown disturbances $W_h$.
    \item If only a noisy estimate $\hat{f}$ of $f$ with $\Vert \hat{f} - f \Vert_\infty \leq \zeta/2$ is available, we show that our algorithm is able to learn (with high probability) an $\mathcal{O}(L\zeta H^2)-$optimal policy in a number of samples that is independent of the cardinalities of state and action spaces. Here, $L$ is the Lipschitz constant of the optimal value function, which characterizes its smoothness.
    \item We further show that if an asymptotically accurate online estimator that can generate a sequence of functions $\{\hat{f}_i\}_{i=1}^K$ with $\Vert \hat{f}_i - f \Vert_\infty \leq \mathcal{O}(\sqrt{d/i})$ exists, then the regret of using such estimators with our algorithm is $\tilde{\mathcal{O}}(\sqrt{H^6T} + L\sqrt{HdT})$.
\end{itemize}

To put these results in perspective, we compare our method with different problem-agnostic Q-learning algorithms in Table~\ref{table:complexity}. Our algorithm effectively incorporates the structure in a way that trades off regret and computation without requiring the $S^2$ space and time dependency of model-based methods.

\section{Related Literature}
The RL literature spans decades and is extensive. This section focuses on a concise review of the most relevant works to our study.

\subsubsection{RL without Simulator.} In the finite episodic setting without access to a simulator, the best regret bound $\tilde{\mathcal{O}}(\sqrt{H^2SAT})$ is achieved by both model-based \cite{azar2017minimax} and and model-free \cite{li2021breaking} Q-learning algorithms. These results are known to be minimax-optimal, and are thus, constrained by the worst-case scenarios. See \citet{domingues2021episodic} for the lower bound proofs and the construction of worst-case MDPs. Problem-dependent bounds characterize sample complexity by additional metrics, such as the maximum per-step conditional variance \cite{zanette2019tighter}, minimum sub-optimality gap \cite{yang2021q}, or the gap-visitation complexity \cite{wagenmaker2022beyond}, and show sharper bounds when these metrics are small. Still, these bounds often show the same polynomial scaling in the cardinalities of states and actions as the problem-free bounds.

\subsubsection{RL with Approximate Model/Simulator.} The use of an approximate model or an approximate simulator in the training of RL methods, sometimes called \emph{sim-to-real transfer}, was shown to effectively improve the sample complexity both empirically \cite{bousmalis2018using, zhao2020sim} and theoretically \cite{chen2021understanding}. Our work is perhaps closest to this line of work in the literature. In \cite{jiang2018pac}, the author assumes access to a probabilistic description of an approximate transition kernel, and, under some assumptions, they characterize the sample complexity by the number of incorrect state-action pairs in the approximate kernel. In an alternative approach, the work of \cite{ayoub2020model} assumes that the actual transition kernel belongs to a known family of models and provides an algorithm with regret that scales with the Eluder dimension for this family. Both works show interesting results where the sample complexity does not scale with the number of states and actions. On the other hand, the algorithms are computationally intensive, requiring a dynamic programming (DP) solver in each iteration. Our work differs from this literature as we do not require an oracle DP solver, nor do we require a detailed probabilistic description of the approximate model.

\subsubsection{RL with System Identification.} Rather than depending on a prior approximate model, other studies in the literature alternate between learning the model (i.e., system identification) and planning with RL \cite{nagabandi2018neural, zhang2019solar, schrittwieser2021online}. In terms of theoretical analysis, the most prominent examples are for the Linear-quadratic regulators (LQRs), where the sample complexity of the planning under an unknown model has been thoroughly quantified by the intrinsic measures of the difficulty of learning linear systems \cite{dean2018regret,cassel2020logarithmic}. Algorithms for LQRs heavily rely on structural properties (e.g., the fact that the optimal policy is linear and the optimal value function is quadratic) in their design and cannot be immediately generalized to a broader class of problems.

\begin{algorithm*}[t]
\caption{ $\big(\text{UCB-f}\big)$ Optimistic Q-learning with $\hat{f}$}
\label{alg:UCB_f}
\begin{algorithmic}[1]
    \State Initialize $Q_h^1(s, a) \leftarrow H$ and $V_h^1(s) \leftarrow H$ for all $(s, a, h) \in \mathcal{S} \times \mathcal{A} \times [H]$
    \For{episode $k = 1, \dots, K$}
        \State Receive $s_1^k$
        \Comment{Initial state}
        \For{step $h = 1, \dots, H$}
            \State $a_h^k \leftarrow \argmax_{a'} Q_h^k(s_h^k, a')$
            \Comment{Greedy action with respect to $Q_h^k$}
            \State Take action $a_h^k$ and observe next state $s_{h+1}^k$
            \State $\hat{w}_h^k \leftarrow s_{h+1}^k - \hat{f}(s_h^k, a_h^k)$ \label{line_1}
            \Comment{Transition randomness}
            \ForAll{$s \in \mathcal{S}$}
                \ForAll{$a \in \mathcal{A}$}
                    \State $s' \leftarrow \hat{f}(s,a) + \hat{w}_h^k$ \label{line_2}
                    \Comment{Simulated next state}
                    \State $Q_h^{k+1}(s,a) \leftarrow (1-\alpha_k)Q_h^k(s,a) + \alpha_k \left[r_h(s,a) + V_{h+1}^k(s') + b_k\right]$
                    \Comment{Q-learning update}
                \EndFor
                \State $V_h^{k+1}(s) \leftarrow \min\left\{ H,\ \max Q_h^{k+1}(s, \cdot) 
                    \right\}$
                    \Comment{Value function update}
            \EndFor
        \EndFor
    \EndFor
\end{algorithmic}
\end{algorithm*}

\section{Preliminaries}

We consider an episodic MDP described by the tuple $(\mathcal{S}, \mathcal{A}, H, \mathbb{P}, \mu, r)$, where $\mathcal{S}=\{1,2,\dots, S\} \subseteq \mathbb{N}$ and $\mathcal{A}=\{1,2,\dots, A\}  \subseteq \mathbb{N}$ are the finite state and the action spaces, $H$ is the number of steps in each episode, $\mathbb{P}$ is the transition kernel such that $\mathbb{P}_h(\cdot | s,a )$ is the distribution over next states if action $a$ is taken at state $s$ and step $h$, $\mu$ is the distribution of initial states, and $r_h: \mathcal{S} \times \mathcal{A} \rightarrow [0,1]$ are the deterministic reward functions. For convenience, we define $[\mathbb{P}_h V](s,a) \coloneqq\mathbb{E}_{ s' \sim \mathbb{P}_h(\cdot | s,a )} [V(s')]$ for any $V: \mathcal{S} \rightarrow \mathbb{R}$.

We use $\pi$ to denote a stationary deterministic policy, which is a collection of $H$ functions $\{ \pi_h: \mathcal{S} \rightarrow \mathcal{A} \}_{h\in[H]}$\footnote{For any integer $n \in \mathbb{Z}^+$, we define $[n] \coloneqq \{1,2,\dots,n\}$.}. The value function $V_h^\pi : \mathcal{S} \rightarrow \mathbb{R}$ and action-value function $Q_h^\pi : \mathcal{S} \times \mathcal{A} \rightarrow \mathbb{R}$ of a policy $\pi$ are defined as:
{\fontsize{8pt}{10.8pt} \selectfont
\begin{equation}
    V_h^\pi(s) \coloneqq \mathbb{E} \left[ \sum_{t=h}^H r_t(s_t, \pi_t(s_t))\ \Big|\ s_h {=} s \right]
\end{equation}
\begin{equation}
    Q_h^\pi(s,a) \coloneqq  r_h(s,a) + \mathbb{E} \left[ \sum_{t=h+1}^H r_t(s_t, \pi_t(s_t))\ \Big|\ s_h {=} s, a_h {=} a \right]
\end{equation}
}
Further, we adopt the convention that $V_{H+1}^\pi(s) = 0$ for any policy. Since the state space, action space, and horizon are all finite and rewards are bounded, there always exists an optimal policy $\pi^\star$ that achieves the optimal value $V_h^\star(s) = \sup_\pi V_h^\pi(s) \ \text{for all } (s,h) \in \mathcal{S} \times [H]$ \cite{bertsekas2017dynamic}. For all $(s, a, h) \in \mathcal{S} \times \mathcal{A} \times [H]$, the Bellman equations for a policy $\pi$ and Bellman optimality equations are:
{\small
\begin{equation}
\begin{split}
    V_h^\pi(s) = Q_h^\pi(s, \pi_h(s)),\mkern9mu Q_h^\pi(s, a) = (r_h {+} \mathbb{P}_h V_{h+1}^\pi)(s,a) \\
    V_h^\star(s) = \max_{a' \in \mathcal{A}} Q_h^\star(s, a'),\mkern9mu Q_h^\star(s, a) = (r_h {+} \mathbb{P}_h V_{h+1}^\star)(s,a)
\end{split}
\label{eq:bellman}
\end{equation}
}

We say that a policy $\pi$ is $\epsilon$-optimal if $V_1^{\pi}(s) \geq V_1^\star(s) - \epsilon$ for all $s \in \mathcal{S}$. Suppose an agent interacts with the MDP for a total of $K$ episodes. The total number of steps is $T\coloneqq HK$. The agent selects a policy $\pi^k$ at the start of each episode $k \in [K]$, receives an initial state $s_1^k$ drawn from the initial state distribution $\mu$, and uses $\pi^k$ to generate all actions in episode $k$. We define the cumulative regret for the agent as:
\begin{equation}
    \text{Regret}(K) \coloneqq \sum_{k=1}^K \left[ V_1^\star(s_1^k) - V_1^{\pi^k}(s_1^k) \right]
\label{eq:regret}
\end{equation}

We restrict the transition kernel $\mathbb{P}$ such that the state transitions are governed by $S_{h+1} = f(S_h, A_h) + W_h$, where $f: \mathcal{S} \times \mathcal{A} \rightarrow \mathcal{S}$ is a deterministic function and $W_h$ are random variables with bounded support $\{0,1,\dots, W\}$ (i.e., $|W_h| \leq W$) generated independently of states and actions. An approximation $\hat{f}$ of $f$ is a function $\mathcal{S} \times \mathcal{A} \rightarrow \mathcal{S}$ that satisfies $\Vert\hat{f}-f\Vert_\infty = \sup_{(s,a)} |\hat{f}(s,a) - f(s,a)| \leq \zeta/2$.

\section{Main Results}
In this section, we present our main algorithm and its formal sample complexity guarantees. As we will see shortly, our algorithm is able to achieve $\sqrt{T}$ regret without dependency in $S$ and $A$ when $f$ is known ($\zeta = 0$). In the presence of persistent approximation errors, the algorithm accumulates additional regret that is upper-bounded by a term of the form $\mathcal{O}(L\zeta H T)$. In the case that approximation errors are nonzero but asymptotically diminishing, the linear term can be suppressed and $\sqrt{T}$ scaling of regret is restored.

\subsection{Algorithm}
We begin by providing the intuition behind our algorithm. The state transitions in Equation~\ref{eq:system} implies that for all $s', s_1, s_2 \in \mathcal{S}, a_1, a_2 \in \mathcal{A},$ and $h \in [H]$, we have:
{
\fontsize{9pt}{10.8pt} \selectfont
\begin{align}
    \mathbb{P}_h(s' | s_1, a_1 ) = \mathbb{P}_h(s'| s_2, a_2) \enspace \textit{if} \enspace f(s_1, a_1) = f(s_2, a_2)
\label{eq:transition_structure}
\end{align}
}
Thus, if one observes a transition from $(s_h,a_h)$ to $(s_{h+1})$, the structure of the transition probabilities suggests that the Q-learning update can be extended to all $(s,a) \in \mathcal{S} \times \mathcal{A}$ pairs according to:
\begin{align}
\begin{split}
    \hat{Q}_h(s,a) \leftarrow & (1-\alpha)\hat{Q}_h(s,a)\ +\\ &\alpha \left[ r_h(s,a) + \hat{V}_{h+1}\big(f(s,a) + w_h\big) \right]
\end{split}
\end{align}
where $w_h = s_{h+1} - f(s_h, a_h)$ and $\alpha$ is a learning rate. One interpretation of the above scheme is that $f$ is used to simulate transitions for state-action pairs not vised in the real trajectories. We contrast this with the typical asynchronous Q-learning that \emph{only} update $(s_h,a_h)$:
{\fontsize{8pt}{10.8pt} \selectfont
\begin{equation}
    \hat{Q}_h(s,a) \leftarrow
    \left\{
    \begin{aligned}
       & (1-\alpha)\hat{Q}_h(s,a)\ + \\
       & \alpha \left[ r_h(s,a) {+} \hat{V}_{h+1}(s_{h+1}) \right], & \textit{if}\ (s{,}a) {=} (s_h{,}a_h)
    \\[1ex]
       & \hat{Q}_h(s,a),                                            & \textit{otherwise}
    \end{aligned}
    \right.
\end{equation}
}

In Algorithm~\ref{alg:UCB_f}, we present an optimistic implementation of the previous scheme, where bonuses are added to the Q-update to incentivize exploration. The exact bonus values will be given in the formal results as they depend on the context. Moreover, there is a specific choice for the learning rate, and it should be set to $\alpha_t = \frac{H+1}{H+t}$. We will discuss the importance of this choice in the next section as well. We also note that all the value functions in Algorithm~\ref{alg:UCB_f} can be updated in place, and the dependency on $k$ (i.e, $Q^{k+1} \shortleftarrow Q^k$) is only for clarity of the analysis.

While the updates carried by Algorithm~\ref{alg:UCB_f} might at first seem similar to synchronous Q-learning \cite{li2021tightening}, there are fundamental differences. In synchronous Q-learning, the algorithm would receive independent samples for every state-action pair, and no cost (i.e., regret) is incurred during this process. On the other hand, our algorithm only receives samples from the trajectory of its greedy policy and accumulates regret while it explores. Finally, our algorithm can be implemented in a model-based fashion, with explicit parameterization of transition probabilities. This could enable more complex transition structures than the one in Equation~\ref{eq:transition_structure} to be encoded in the Q-updates, but it is unnecessary in our case, so we choose the model-free implementation for its computational advantages.

\subsection{Sample Complexity}

In general, we need to assume some regularities in the value functions for learning with a noisy estimator to be feasible. In our analysis, we assume Lipschitz continuity. That is, for all $s_1, s_2 \in \mathcal{S}$ and $h \in [H]$, we have\footnote{ Note that $L=H$ holds trivially for any MDP in our setting; nonetheless, one would hope that $L \ll \mathcal{O}(H)$. An upper bound on $L$ can be derived from Lipschitz constants of $\hat{f}$ and $r$ without access to the optimal value functions \cite{hinderer2005lipschitz}.}:
\begin{equation}
    | V_h^\star(s_1) - V_h^\star(s_2) | \leq L | s_1 - s_2 |.
\label{eq:lipschitz_assumption}
\end{equation}
In the following we assume that Algorithm~\ref{alg:UCB_f} is run with a noisy estimate $\hat{f}$ that satisfies:
\begin{equation}
    \Vert \hat{f} - f \Vert_\infty \leq \zeta/2
\label{eq:noisy_model}
\end{equation}
where $\zeta$ is known. With these, we are ready to state our formal regret guarantee.

\begin{theorem}
For any $p \in (0,1)$, let the bonus term applied for all Q-learning updates during the $k$-th episode be $b_k= c \sqrt{\frac{H^3\log(SAH/p)}{k}} + L\zeta$. Then with probability $1{-}p$, the total regret of Algorithm~\ref{alg:UCB_f} when applied to the class of MDPs in Equation~\ref{eq:system}, is at most $\mathcal{O}(\sqrt{H^6T\iota} + L\zeta HT)$, where $c>0$ is an absolute constant and $\iota = \log(SAH/p)\log{K}$.
\label{thm:regret}
\end{theorem}

A few comments are in order. The theorem implies that when $\zeta=0$, we are able to achieve $\sqrt{T}$ regret with only a logarithmic dependency on the cardinalities of states and actions. When $\zeta > 0$, a linear $\zeta T$ term in the regret is unavoidable due to the persisting bias from using $\hat{f}$ instead of the true function $f$ (see, e.g., \citet{jin2020provably} for similar results under model mismatch). A key observation is that our algorithm does not amplify this term by any factor that depends on the number of states and actions. Besides regret, our algorithm also enjoys direct guarantees on its policies.

\begin{theorem}
Under the same conditions of Theorem~\ref{thm:regret}, Algorithm~\ref{alg:UCB_f} will generate a policy that is $(\epsilon + \mathcal{O}(L\zeta H^2))$- optimal with probability $1-p$, when the number of episodes $K$ is at least $\mathcal{O}\left( \frac{H^7 \iota}{\epsilon^2} \right)$, where $\iota = \log(SAH/p)$.
\label{thm:bpi}
\end{theorem}

Similar to Theorem~\ref{thm:regret}, the policy converges to the optimal policy when $\zeta=0$, and to a sub-optimal policy when $\zeta>0$. Further, we believe that the dependency on $H$ is not sharp and can be improved (see, the Discussion section). It is worth noting that Theorem~\ref{thm:bpi} does not rely on heuristic translations of regret to best policy identification. These translations typically have suboptimal $1/{p^2}$ dependency and only hold with constant probability \cite{menard2021fast}. Whereas, Theorem~\ref{thm:bpi} has the optimal $\log(1/{p})$ scaling and holds with high probability. 

Finally, we state the results for the case where the estimates $\hat{f}$ can be improved over time in an online fashion.

\begin{theorem}
Suppose that there exists an online estimator that generates a sequence of functions $\{\hat{f}_i\}_{i=1}^K$ such that $\Vert \hat{f}_i - f \Vert_\infty \leq \mathcal{O}(\sqrt{d/i})$ holds uniformly with high probability, then the regret of using such estimators with Algorithm~\ref{alg:UCB_f} is at most $\tilde{\mathcal{O}}(\sqrt{H^6T} + L\sqrt{HdT})$.
\label{thm:online_estimator}
\end{theorem}

In general, without any assumptions on $f$, the parameter $d$ could scale with $\mathcal{O}(SA)$, and the results in Theorem~\ref{thm:online_estimator} would match the worst-case bounds in terms of the scaling with the number of states and actions. On the other hand, if $f$ is parameterized with few parameters or belongs to a family of functions where efficient online estimators exist, the theorem implies that the sample complexity is tied directly to the complexity of learning $f$. An important point to mention is that the online estimator needs to work from the data generated by the greedy policies $\{\pi^i\}_{i=1}^K$. If the $\mathcal{O}(K)$ switching of policies is too fast for the online estimator, one could exponentially reduce the switching of policies without any large impact on the regret (see, e.g., the work of \citet{bai2019provably} for a $\mathcal{O}(\log{K})$ switching and the work of \citet{qiao2022sample} for a $\mathcal{O}(\log{\log{K}})$ switching).

\section{Proofs Sketch}
In this section, we provide the intuition behind the proof of our results. For clarity, we will focus on the noiseless case $\zeta=0$. The full proofs for the general case $\zeta>0$ are provided in the appendix. We use $(s_h^k, a_h^k, w_h^k)$ to denote the actual state, action, and randomness observed or chosen at step $h$ of episode $k$. In addition, we use $\pi_h^k$ to denote the greedy policy with respect to $Q_h^k$. Notice that $a_h^k = \pi_h^k(s_h^k)$. Further, we define an empirical transition operator at step $h$ of episode $k$ as $[\mathbb{\hat{P}}_h^k V](s,a) \coloneqq V\big(f(s,a)+w_h^k\big)$ for all $(s,a) \in \mathcal{S} \times \mathcal{A}$. Moreover, we use $c>0$ to denote an absolute constant and $\iota \coloneqq \log(SAH/p)$ to denote a log factor.

\subsection{$(Q_h^k - Q_h^\star)$ Decomposition}
To start, we define useful quantities for the learning rate that will help us expand the updates carried out by Algorithm~\ref{alg:UCB_f}:
{\fontsize{9pt}{10.8pt} \selectfont
\begin{talign}
\begin{split}
    \alpha_t^0 \coloneqq (1{-}\alpha_1)\cdots(1-\alpha_t) &= \prod_{j=1}^t (1-\alpha_j) \\
    \alpha_t^i \coloneqq \alpha_i(1-\alpha_{i+1})\cdots(1-\alpha_t) &= \alpha_i \prod_{j=i+1}^t (1-\alpha_j)
\end{split}
\label{eq:alpha_i_t}
\end{talign}
}
$\alpha_t^0$ can be understood as the weight of the Q-value initialization in the current estimate of Q-values after $t$ total updates, while $\alpha_t^i$ as the weight of the $i$th update. With these definitions, we have for $t = k-1 \geq 1$:
{\fontsize{8pt}{10.8pt} \selectfont
\begin{equation}
\begin{split}
    (Q_h^k - Q_h^\star)(s,a) = \sum_{i=1}^{t} \alpha_t^i \Big[& \big(V_{h+1}^i - V_{h+1}^\star\big)(f(s,a) + w_h^i) \\ 
    &+  \big(\mathbb{\hat{P}}_h^i - \mathbb{P}_h\big)V_{h+1}^\star (s,a) + b_i \Big]
\end{split}
\label{eq:Q_error}
\end{equation}
}

There are three sources of errors in the previous equation: The bootstrapping with $V^i_{h+1}$ instead of $V^\star_{h+1}$, using samples $\mathbb{\hat{P}}_h^i$ to approximate the true expectation $\mathbb{P}_h$, and finally the bias due to adding optimism bonuses. Notice that, from concentration \cite{dubhashi2009concentration}, we know that the variance of $(\mathbb{\hat{P}}_h^i - \mathbb{P}_h)$ will scale with $\mathcal{O}(1/\sqrt{i})$ and the bonuses should be designed to have the same scaling. Further, $(V^i_{h+1} - V_{h+1}^\star)$ is related to $(Q_{h+1}^i - Q_{h+1}^\star)$, which would allow us to accumulate the errors over $h$ by recursion.

A key challenge in model-free Q-learning is controlling the propagation of errors over the $H$ steps. To illustrate this point, consider if we used a typical learning rate of $\eta_t = 1/t$. \linebreak In this case, we have $\eta_t^i \coloneqq (1{-}\eta_1)\cdots(1-\eta_t) = 1/t$ and $\sum_{i=1}^t \frac{\eta_t^i}{\sqrt{i}} \approx \sqrt{\frac{\log{t}}{t}}$. If we then do a recursion over $h$ we would end up with $(Q_1^k - Q_1^\star) \approx \frac{(\log{k})^{\sfrac{H}{2}}}{\sqrt{k}}$, which is \linebreak exponential in $H$. The following lemma, which is a fine-grained version of the one shown in \cite{jin2018q}, shows that we can avoid this exponential scaling with the rescaled linear rate $\alpha_t = \frac{H+1}{H+t}$.
\begin{lemma}
The following holds for $\alpha_t^i$:
{
\begin{equation*}
    \frac{1}{t^a} \leq \sum_{i=1}^t \frac{\alpha_t^i}{i^a} \leq \frac{1 + \frac{1}{H}}{t^a} \quad \text{for every} \ t \geq 1 \ \text{and} \ a \in \left[ \tfrac{1}{2}, 1 \right]
\end{equation*}
}
\label{lem:alpha_i_t}
\end{lemma}

\subsection{Convergence of On-Policy Errors}
Then, using Azuma-Hoeffding inequality, we choose $b_t$ such that the following holds simultaneously for all $(s,a,h,k) \in \mathcal{S} \times \mathcal{A} \times [H] \times \{2,\dots,K\}$ with high probability:
{
\fontsize{9pt}{10.8pt} \selectfont
\begin{equation}
0 \leq \sum_{i=1}^t \alpha_t^i \Bigg[ \left(\mathbb{\hat{P}}_h^i - \mathbb{P}_h\right)V_{h+1}^\star (s,a) + b_i \Bigg] \leq  2c\sqrt{\frac{H^3\iota}{t}}
\label{eq:concentration}
\end{equation}
}

Now, starting from Equation~\ref{eq:Q_error} and equipped with Lemma~\ref{lem:alpha_i_t} and the upper bound in Equation~\ref{eq:concentration}, we construct a careful recursive argument that shows a convergence of the Q-value on-policy (i.e., for the actions chosen by $\pi^k$) errors:
\begin{lemma}
Let $\left\{ P_h \right\}_{h=1}^{H}$ be a sequence that satisfy the following recursive relationship:
{
\begin{equation*}
    P_h = 4 \left(1+\frac{1}{H}\right) c\sqrt{H^3\iota} + \left(1+\frac{1}{H}\right)P_{h+1}
\end{equation*}
}
where $P_{H+1} \coloneqq 0$. Then, for any $p \in (0,1)$, letting $b_t = c \sqrt{H^3\iota/t}$, the following holds simultaneously for all $(s,h,k) \in \mathcal{S} \times [H] \times \{2,\dots,K\}$ with at least $1-p$ probability:
{
\begin{equation*}
    0 \leq (V_h^k - V_h^\star)(s) \leq (Q_h^k - Q_h^\star)(s, \pi_h^k(s)) \leq \frac{P_h}{\sqrt{t}}
\end{equation*}
}
where $t=k-1$.
\label{lem:bound_on_policy}
\end{lemma}

Lemma~\ref{lem:bound_on_policy} states that the value function estimates are optimistic and converge with a $1/\sqrt{k}$ rate. Importantly, the statement holds uniformly for all $s \in \mathcal{S}$. Note that the error in $Q_h^1$ is determined by the initialization, and thus, is excluded from the lemma.
Finally, the elements $\left\{ P_h \right\}_{h=1}^{H}$ are upper bounded by a function that is polynomial\footnote{The the elementary fact $\left(1+\frac{1}{x}\right)^x \leq e$ is helpful here.} in $H$ and and log-polynomial in $S$ and $A$:
{
\begin{equation}
    P_h \leq 4(1+H-h) c\sqrt{H^3\iota} \left(1+\frac{1}{H}\right)^{1+H-h}
\label{eq:P_h_bound}
\end{equation}
}

We note that the argument presented here is different than the one used in previous literature (e.g., in \citet{jin2018q} or \citet{li2021breaking}), where the propagation of errors over the $H$ steps is controlled after summing over $K$ and using different property of $\alpha_t^i$. We cannot apply the same argument here as it requires that the Q-values are updated only for a \emph{single} trajectory per episode, while in our case all the Q-values are updates per episode. Our argument exploits the uniformity of the number of updates for every $(s,a)$ pair to show convergence before summing over $K$.

\subsection{Regret and Best Policy Identification}
Now, we bound the suboptimality gap $(V_h^\star - V_h^{\pi^k})$. Following the standard techniques for optimum-based methods, we start by showing:
{
\begin{align*}
    \big( V_h^\star - V_h^{\pi^k} \big) (s) \leq & \big( V_h^k - V_h^{\pi^k} \big) (s)  \\
    \leq & \big( Q_h^k - Q_h^\star \big)(s, \pi_h^k(s)) \\
    & + \mathbb{P}_h \big( V_{h+1}^\star - V_{h+1}^{\pi^k} \big) (s, \pi_h^k(s))
    \numberthis \label{eq:suboptimality_1}
\end{align*}
}
The first equality follows as $V_h^k \geq V_h^\star$ and the second from Bellman Equations~\ref{eq:bellman} after adding and subtracting $Q_h^\star$. Notice that the first term in Equation~\ref{eq:suboptimality_1} is upper bounded by Lemma~\ref{lem:bound_on_policy} and the second term is an expectation over suboptimality gap in the next step $h+1$. Thus, we can establish by recursion that the following holds with high probability:
{
\begin{align}
    \big( V_1^\star - V_1^{\pi^k} \big) (s) \leq \sum_{h=1}^H \frac{P_h}{\sqrt{k-1}} \leq \mathcal{O}\left( \sqrt{\frac{H^7\iota}{k-1}} \right) \label{eq:suboptimality_2}
\end{align}
}
This is sufficient to show convergence to optimal policies, and regret follows immediately by summing over $K$ and applying Jensen's inequality. The proofs for $\zeta>0$ follow the same recipe presented here but are more elaborate as the bonus design needs to guarantee that the estimates $Q_h^k$ are optimistic even with the errors in $\hat{f}$.

\begin{figure*}%
    \centering
    \includegraphics[width=7in]{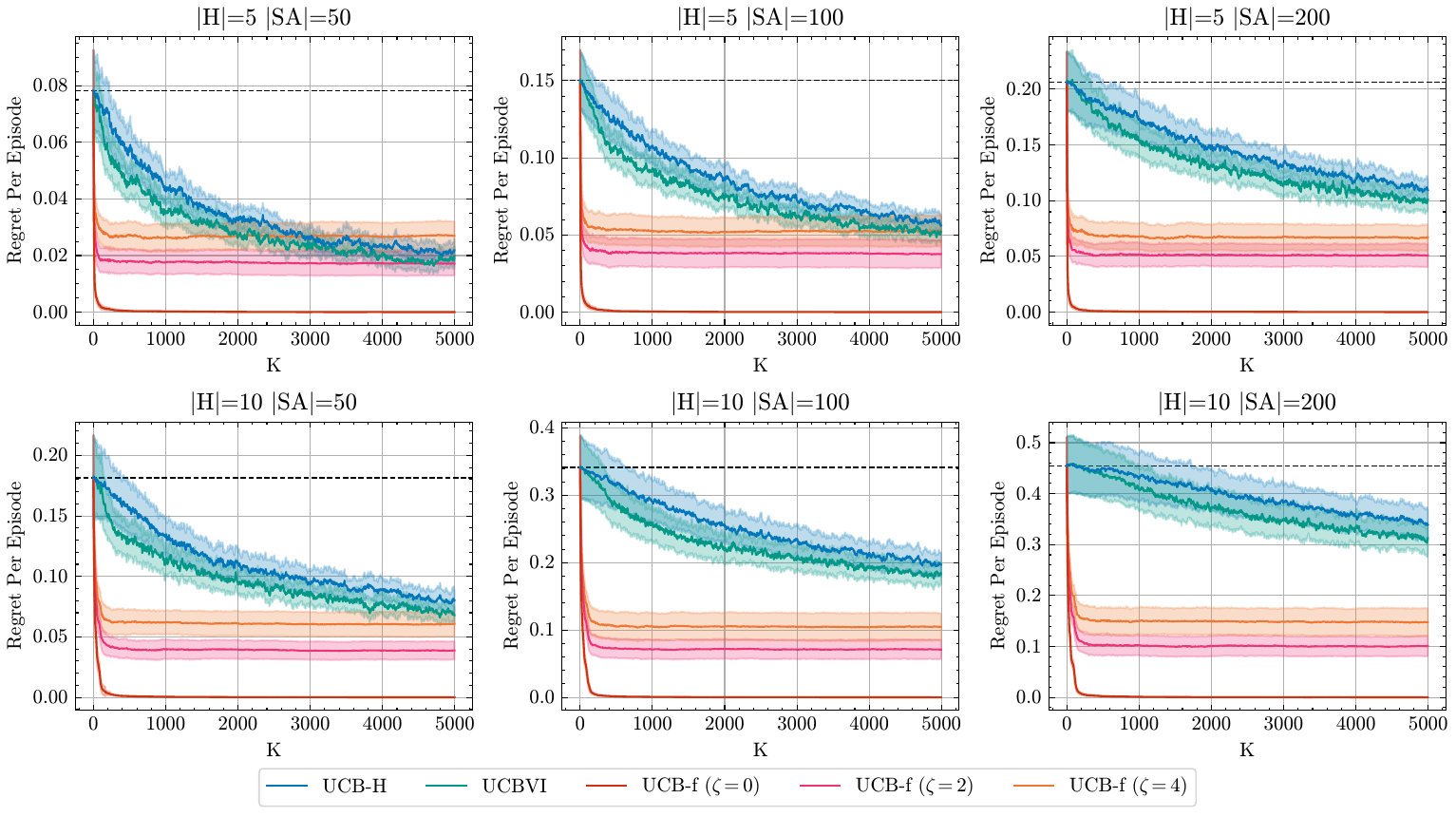}
    \caption{$(V_1^\star - V_1^{\pi^k})$ for different Q-learning algorithms on randomly generated MDPs. Each curve is an average of 50 simulations. The dashed lines are the performance of the greedy policy with respect to the rewards.}%
    \label{fig:regret_compare}
\end{figure*}

\section{Experiment}
In this section, we empirically compare our algorithm with multiple problem-agnostic Q-learning algorithms from the literature. In particular, we compare with the model-free UCB-H method of \citet{jin2018q} and the model-based UCBVI method of \citet{azar2017minimax}. When we test Algorithm~\ref{alg:UCB_f}, we either provide it the true function $f$ (to test $\zeta=0$) or we corrupt $f$ with random noise (to test $\zeta>0$) before we give it to the algorithm. Inside a single run, the approximate function $\hat{f}$ is kept fixed across the time $T$.

\subsubsection{Environments.}
We compare the algorithms on randomly generated MDPs with varying cardinalities. Further, as our method and UCBVI require the rewards to be known, we choose to make $r_h(s,a) = r_h(s)$. This ensures that there is no utility in being greedy with respect to the rewards, and the optimal policy is the one that controls the system toward a desired region in the state space.

\subsubsection{Bonuses.}
As multiplicative constants in the bonuses can drastically change empirical performances, we unify the bonus design across the different methods. Specifically, we use the following bonuses in the empirical simulations:
{
\begin{align}
    \beta_1 = c\sqrt{\frac{H^2}{N_h^k(s,a)}},\quad \beta_2 = c\sqrt{\frac{H^2}{k}} + c\zeta L
\end{align}
}
UCB-H and UCBVI use $\beta_1$, and our method uses $\beta_2$. $N_h^k(s,a)$ is the state-action visitation counter used by the problem-agnostic methods. Further, as our goal is to study qualitative behavior, we choose to optimize $c$ by setting it $c=0.05$. This would accelerate the convergence of all algorithms, especially UCB-H and UCBVI as their visitation counter is distributed over the state-action pairs.

\subsubsection{Results.} In Figure~\ref{fig:regret_compare}, we plot the regret per episode (i.e., the suboptimality gap) for $S=25$, $A=\{2,4,8\}$, $H=\{5,10\}$, $W=5$, $\zeta=\{0,2,4\}$, $L=0.25$, and $K\in[5000]$. Each curve is an average of 50 simulations, and the shadings represent a width of one standard deviation. Moreover, the dashed lines represent the performance of the greedy policy with respect to rewards. Because of our choice for $r_h$, the rewards-greedy policy is no better than the random policies that the algorithms start with.

As we can see from Figure~\ref{fig:regret_compare}, Algorithm~\ref{alg:UCB_f} always converge to the optimal policy (when $\zeta=0$) or to a suboptimal policy (when $\zeta>0$). Moreover, the time until convergence is not affected by increasing the cardinalities of states and actions. This is in contrast to UCB-H and UCBVI, where it is evident from the figure that their convergence is affected by $S$ and $A$, as they take longer to outperform the case of $\zeta>0$ when the cardinalities increase. Finally, we note that the variance in the figure when $\zeta>0$ is not from within run variations but is from the randomness of generating a new MDP in each simulation. The code to reproduce these results is publicly available\footnote{\url{https://github.com/meshal-h/ucb-f}}.

\section{Discussion}

In this section, we elaborate on some aspects of our algorithm and methodology and discuss potential extensions of our study.

\subsubsection{Dependency on the Horizon.}
We believe that the scaling with the horizon in our results is not optimal, and there is an extra $H$ factor that is an artifact of the analysis. Further, sharpening the horizon dependency by a $\sqrt{H}$ factor is possible with a more elaborate bonus design that relies on Bernstein inequality. Moreover, variance reduction techniques were shown to yield another $\sqrt{H}$ factor improvement \cite{li2021breaking}. We note that using these techniques usually adds lower-order terms to the regret (see, Table~\ref{table:complexity}). As our focus in this paper is on the dependency on $S$ and $A$, we leave applying these ideas to our algorithm as future directions.

\subsubsection{Value Function Assumptions.}
As shown by \citet{jiang2018pac}, without further assumptions, no algorithm can always achieve sample complexity independent of the number of states and actions, even if it has access to a transition kernel that is only corrupted at one state-action pair. We note that our assumptions in Equation~\ref{eq:lipschitz_assumption} only concern the optimal value functions $V_h^\star(s)$, and no smoothness is required for the value function generated by non-optimal policies. For future analysis, we could consider more local regularity assumptions on the optimal value function or adapt some of the metrics used by the problem-dependent bounds discussed in the literature review.

\subsubsection{Noise in the Model.}
Not all noise structures in $\hat{f}$ are equivalent. In fact, some can be completely benign; The current analysis suggests that if the approximate function differs from the true function by a constant (i.e., $\hat{f} = f + c$ for all state-action pairs), then there is no linear term in the regret when we use $\hat{f}$ (owing to the cancellation of $c$ that would happen between Line~\ref{line_1} and \ref{line_2} in Algorithm~\ref{alg:UCB_f}). This would suggest that the use of $\Vert \hat{f} - f \Vert_\infty$ in the analysis can be refined, which will also provide relaxations to the conditions required for the online estimator in Theorem~\ref{thm:online_estimator}.

\subsubsection{Systems that Fit the Additive Disturbance Model.} 
Perhaps the simplest example is inventory control $S_{h+1} = S_{h} + A_{h} - W_{h}$, where $S_{h}$ is the inventory level, $A_{h}$ is the amount of goods ordered, and $W_{h}$ is the demand observed after taking action $A_{h}$. In this model, the state can be negative to represent backlogged demand. We note that although we restricted the states, actions, and disturbances to be positive, there is no loss of generality and our analysis extends to the case where negative values are allowed. The additive disturbance model in our study is flexible and captures systems such as inventory control and many other similar problems that arise in operations research. Further, as the additive disturbance model is prevalent in control problems with continuous spaces, we discuss the implications of our results on these systems in the following subsection.

\subsubsection{Continuous Spaces.} 
With the appropriate topological assumptions, Algorithm~\ref{alg:UCB_f} can be applied to continuous problems after discretization and inheriting the metric (to measure $f - \hat{f}$) from the continuous space (see, e.g., \citet{shah2018q} for an example of Q-learning with discretization). Our results would suggest that we will be able to learn an optimal policy (up to $\Delta$ discretization errors) when $\zeta=0$ in a number of samples that do not scale with state and action cardinalities, which is $\mathcal{O}\left(\frac{1}{\Delta^2}\right)$ in this case. A caveat here is that computation will also scale with $\mathcal{O}\left(\frac{1}{\Delta^2}\right)$ (cf. Table~\ref{table:complexity}). Still, one might be able to modify Algorithm~\ref{alg:UCB_f} to make it more practical under sufficient smoothness conditions. A more interesting line of future work is to see if our methodology can be extended to the value function approximation setting, which is more appropriate in continuous spaces.

\section{Conclusion}
In this paper, we studied the sample complexity of online RL for systems that evolve according to an additive disturbance model of the form $S_{h+1} = f(S_h, A_h) + W_h$. We developed an optimistic Q-learning algorithm that achieves $\sqrt{T}$ regret without dependency in $S$ and $A$ when $f$ is known. Further, if an asymptotically accurate online estimator for $f$ exists, we also achieve $\sqrt{T}$ regret that depends only on the complexity of learning $f$. Our methodology accomplishes all of that without explicitly modeling transition probabilities, and our algorithm enjoys the same memory complexity as model-free methods. Future research directions include extending our algorithm and analysis to infinite horizon MDPs and problems involving value function approximations.

\section{Acknowledgments}
Meshal Alharbi is supported by a scholarship from King Abdulaziz City for Science and Technology (KACST).

{
\fontsize{9.8pt}{10.8pt} \selectfont
\bibliography{references}
}

%
% Appendix
\clearpage
\appendix
\onecolumn

\section{Properties of $\alpha_t^i$}

We restate Lemma~\ref{lem:alpha} to add other helpful properties for the proofs.

\begin{lem}
The following properties hold for $\alpha_t^i$:
\begin{enumerate}
    \item $\frac{1}{t^a} \leq \sum_{i=1}^t \frac{\alpha_t^i}{i^a} \leq \frac{1 + \frac{1}{H}}{t^a} \ for\ every\ t \geq 1$ and $a \in \left[ \frac{1}{2}, 1 \right]$ \label{lem:alpha_a}
    \item $\max_{i \in [t]} \alpha_t^i \leq \frac{2H}{t} \ and\ \sum_{i=1}^t (\alpha_t^i)^2 \leq \frac{2H}{t} \ for\ every\ t \geq 1$ \label{lem:alpha_b}
    \item $\sum_{t=i}^\infty \alpha_t^i = 1 + \frac{1}{H} \ for\ every\ i \geq 1$ \label{lem:alpha_c}
\end{enumerate}
\label{lem:alpha}
\end{lem}
\begin{proof}
We prove $(1)$ by induction over $t$. For the base case $t=1$, the statement holds since:
\begin{equation*}
    \sum_{i=1}^t \frac{\alpha_t^i}{i^a} = \alpha_1^1 = \alpha_1 = 1 \in \left[1, 1 +\frac{1}{H} \right]
\end{equation*}
Now suppose that $t \geq 2$. From Equation~\ref{eq:alpha_i_t}, we see that $\alpha_t^i = \alpha_{t-1}^i (1 - \alpha_t)$. Thus:
\begin{equation*}
    \sum_{i=1}^t \frac{\alpha_t^i}{i^a} = (1 - \alpha_t)\sum_{i=1}^{t-1} \frac{\alpha_{t-1}^i}{i^a} + \frac{\alpha_t}{t^a}
\end{equation*}
Then, assuming that $\frac{1}{(t-1)^a} \leq \sum_{i=1}^{t-1} \frac{\alpha_{t-1}^i}{i^a} \leq \frac{1 + \frac{1}{H}}{(t-1)^a}$, we have:
\begin{equation*}
    (1 - \alpha_t)\sum_{i=1}^{t-1} \frac{\alpha_{t-1}^i}{i^a} + \frac{\alpha_t}{t^a} 
    \geq \frac{1 - \alpha_t}{(t-1)^a} + \frac{\alpha_t}{t^a} 
    \geq \frac{1 - \alpha_t}{t^a} + \frac{\alpha_t}{t^a}
    = \frac{1}{t^a}
\end{equation*}
On the other hand, we have:
\begin{align*}
    (1 - \alpha_t)\sum_{i=1}^{t-1} \frac{\alpha_{t-1}^i}{i^a} + \frac{\alpha_t}{t^a} 
    &\leq \frac{(1 - \alpha_t)(1 + \frac{1}{H})}{(t-1)^a} + \frac{\alpha_t}{t^a}
    \stackrel{\text{\tiny\Circled{1}}}{=} \frac{(t-1)^{1-a}(1 + \frac{1}{H})}{H+t} + \frac{H+1}{t^a(H+t)} \\
    &\stackrel{\text{\tiny\Circled{2}}}{\leq} \frac{t^{1-a}(1 + \frac{1}{H})}{H+t} + \frac{H+1}{t^a(H+t)}
    = \frac{t + \frac{t}{H} + H + \frac{H}{H}}{t^a(H+t)} = \frac{1+\frac{1}{H}}{t^a}
\end{align*}
where {\tiny\CircledTop{1}} holds from the choice $\alpha_t = \frac{H+1}{H+t}$ and {\tiny\CircledTop{2}} holds since $a \leq 1$. This concludes the induction. The proofs of $(2)$ and $(3)$ are given in Appendix B in \citet{jin2018q}.
\end{proof}

Moreover, it is straightforward to show that for $t=0$, we have $\alpha_t^0 = 1$ and $\sum_{i=1}^t \alpha_t^i = 0$. Further, for $t \geq 1$, we have $\alpha_t^0 = 0$ and $\sum_{i=1}^t \alpha_t^i = 1$.

%\newpage

\section{Sample Complexity Proofs when $\zeta>0$}

\subsection{Notations}

Without loss of generality, we assume that we can write the noisy model as:
\begin{align}
    \hat{f}(s,a) = f(s,a) + \varepsilon_{s,a}
%\label{eq:noisy_model}
\end{align}
where $| \varepsilon_{s,a} | \leq \zeta / 2$. Further, we define the following quantities:
\begin{align*}
    \varepsilon_{s,a}^{h,i} &= \big( f(s,a) - \hat{f}(s,a) \big) - \big( f(s_h^i,a_h^i) - \hat{f}(s_h^i,a_h^i) \big) \numberthis \\
    &= \varepsilon_{s,a} - \varepsilon_{s_h^i,a_h^i} \\
    &\leq \zeta
\end{align*}
We keep using $(s_h^k, a_h^k, w_h^k)$ to denote the realizations of variables and $\pi_h^k$ to denote the greedy policy with respect to $Q_h^k$. Further, we use $c>0$ to denote an absolute constant and $\iota \coloneqq \log(SAH/p)$. On the other hand, we \emph{redefine} the empirical transition operator at step $h$ of episode $k$ to:
\begin{equation}
    [\mathbb{\hat{P}}_h^k V](s,a) \coloneqq V\big( f(s,a) + w_h^k + \varepsilon_{s,a}^{h,k} \big)
\label{eq:empirical_operator_noisy}
\end{equation}
With these notations, the decomposition of $Q_h^k - Q_h^\star$ for any $(s, a, h, k) \in \mathcal{S} \times \mathcal{A} \times [H] \times [K]$ is given by:
\begin{equation}
    (Q_h^k - Q_h^\star)(s,a) = \alpha_t^0 (H - Q_h^\star(s,a)) + \sum_{i=1}^t \alpha_t^i \left[ (V_{h+1}^i - V_{h+1}^\star)(f(s,a) + w_h^i +\varepsilon_{s,a}^{h,i}) + \left[ \left(\mathbb{\hat{P}}_h^i - \mathbb{P}_h\right)V_{h+1}^\star \right](s,a) + b_i\right]
\label{eq:decomposition_noisy}
\end{equation}

\subsection{Optimism}

In the $\zeta > 0$ setting, bonus design needs to guarantee that the estimates $Q_h^k$ are optimistic even with the errors $\varepsilon_{s,a}$. We establish this formally in the following lemma.
\begin{lemma}
For any $p \in (0,1)$, letting $b_t = c\sqrt{\frac{H^3\iota}{t}} + L\zeta$, the following holds simultaneously for all $(s,a,h,k) \in \mathcal{S} \times \mathcal{A} \times [H] \times \{2,\dots,K\}$ with at least $1-p$ probability:
\begin{equation*}
    0 \leq (Q_h^k - Q_h^\star)(s,a) \leq \sum_{i=1}^t \alpha_t^i (V_{h+1}^i - V_{h+1}^\star)(f(s,a) + w_h^i +\varepsilon_{s,a}^{h,i}) + 2 \left( 1 + \frac{1}{H} \right) b_t
\end{equation*}
where $t=k-1$.
\label{lem:optimism}
\end{lemma}
\begin{proof}
To begin, let $\mathcal{F}_i$ be the $\sigma$-field generated by all the random variables until episode $i$. For any $(s,a,h) \in \mathcal{S} \times \mathcal{A} \times [H]$, we define for all $i \in [K]$:
\begin{align*}
    X_i = V_{h+1}^\star(f(s,a) + w_h^i) - \mathbb{P}_h V_{h+1}^\star(s,a)
\end{align*}
Notice that $X_i$ is a martingale difference sequence with respect to the filtration $\{\mathcal{F}_i\}_{i\geq0}$. Moreover, we have $|\alpha_i^t X_i| \leq \alpha_i^t H$. Thus, by Azuma-Hoeffding inequality, we have with probability at least $1-\frac{p}{SAH}$:
\begin{align*}
    \left| \sum_{i=1}^t \alpha_i^t X_i \right| \leq \frac{cH}{2} \sqrt{\iota \sum_{i=1}^{t} (\alpha_i^t)^2} \leq c \sqrt{\frac{H^3\iota}{t}}
\end{align*}
The second inequality follows from Lemma~\ref{lem:alpha_b}. Then, using a union bound on $s,a,$ and $h$, we see that the following holds simultaneously for all $(s,a,h,k) \in \mathcal{S} \times \mathcal{A} \times [H] \times \{2,\dots,K\}$ with at least $1 - p$ probability:
\begin{equation}
    \left| \sum_{i=1}^t \alpha_t^i\left[ V_{h+1}^\star(f(s,a) + w_h^i) - \mathbb{P}_h V_{h+1}^\star(s,a) \right] \right|\leq c \sqrt{\frac{H^3\iota}{t}}
\label{eq:azuma_hoeffding}
\end{equation}
Now, we choose $b_t = c\sqrt{\frac{H^3\iota}{t}} + L\zeta$, which satisfies the following:
\begin{align*}
    \sum_{i=1}^t \alpha_t^i b_i
    \stackrel{\text{\tiny\Circled{1}}}{=} c\sqrt{H^3\iota} \sum_{i=1}^t \frac{\alpha_t^i}{\sqrt{i}} + L \zeta 
    \stackrel{\text{\tiny\Circled{2}}}{\leq} \left( 1 + \frac{1}{H} \right) \left( c \sqrt{\frac{H^3\iota}{t}} + L\zeta \right) = \left( 1 + \frac{1}{H} \right) b_t \numberthis \label{eq:b_t_upper}
\end{align*}
{\tiny\CircledTop{1}} follows because $\sum_{i=1}^t \alpha_t^i = 1$ and {\tiny\CircledTop{2}} from Lemma~\ref{lem:alpha_a}. Similarly:
\begin{align*}
    \sum_{i=1}^t \alpha_t^i b_i \geq \left( c \sqrt{\frac{H^3\iota}{t}} + L\zeta \right) = b_t \numberthis \label{eq:b_t_lower}
\end{align*}
Further, we can show that:
\begin{align*}
    \sum_{i=1}^t \alpha_t^i \left[ \left(\mathbb{\hat{P}}_h^i - \mathbb{P}_h\right)V_{h+1}^\star \right](s,a)
    &\stackrel{\text{\tiny\Circled{1}}}{=} \sum_{i=1}^t \alpha_t^i V_{h+1}^\star\big( f(s,a) + w_h^i + \varepsilon_{s,a}^{h,i} \big) - \mathbb{P}_h V_{h+1}^\star(s,a) \\
    &\stackrel{\text{\tiny\Circled{2}}}{\leq} \sum_{i=1}^t \alpha_t^i V_{h+1}^\star\big( f(s,a) + w_h^i \big) - \mathbb{P}_h V_{h+1}^\star(s,a) + \sum_{i=1}^t \alpha_t^i L | \varepsilon_{s,a}^{h,i} | \\
    &\stackrel{\text{\tiny\Circled{3}}}{\leq} c \sqrt{\frac{H^3\iota}{t}} + L \zeta = b_t \numberthis \label{eq:hat_p_true_p_bound_upper}
\end{align*}
where {\tiny\CircledTop{1}} holds from the definition of $\mathbb{\hat{P}}_h^i$ in Equation~\ref{eq:empirical_operator_noisy}, {\tiny\CircledTop{2}} from the Lipschitz assumption in Equation~\ref{eq:lipschitz_assumption}, and {\tiny\CircledTop{3}} holds with high probability from Equation~\ref{eq:azuma_hoeffding}. Similarly, we can establish the following lower bound:
\begin{align*}
    \sum_{i=1}^t \alpha_t^i \left[ \left(\mathbb{\hat{P}}_h^i - \mathbb{P}_h\right)V_{h+1}^\star \right](s,a)
    \geq -c \sqrt{\frac{H^3\iota}{t}} - L \zeta = -b_t \numberthis \label{eq:hat_p_true_p_bound_lower}
\end{align*}
Then, starting from Equation~\ref{eq:decomposition_noisy}, we establish the right-hand side of the lemma by using the bounds in Equations~\ref{eq:b_t_upper} and \ref{eq:hat_p_true_p_bound_upper}. For any $(s,a,h,k) \in \mathcal{S} \times \mathcal{A} \times [H] \times \{2,\dots,K\}$, we have:
\begin{align*}
    (Q_h^k - Q_h^\star)(s,a) =& \sum_{i=1}^t \alpha_t^i \left[ (V_{h+1}^i - V_{h+1}^\star)(f(s,a) + w_h^i +\varepsilon_{s,a}^{h,i}) + \left[ \left(\mathbb{\hat{P}}_h^i - \mathbb{P}_h\right)V_{h+1}^\star \right](s,a) + b_i\right] \\
    \leq& \sum_{i=1}^t \alpha_t^i (V_{h+1}^i - V_{h+1}^\star)(f(s,a) + w_h^i +\varepsilon_{s,a}^{h,i}) + b_t + \left( 1 + \frac{1}{H} \right) b_t \\
    \leq& \sum_{i=1}^t \alpha_t^i (V_{h+1}^i - V_{h+1}^\star)(f(s,a) + w_h^i +\varepsilon_{s,a}^{h,i}) + 2 \left( 1 + \frac{1}{H} \right) b_t
\end{align*}

Now, we proof the left-hand side of the lemma. First, let $x_{s,a}^{h,i} = f(s,a) + w_h^i +\varepsilon_{s,a}^{h,i}$. Starting from Equation~\ref{eq:decomposition_noisy}, we can show that for any $(s,a,k) \in \mathcal{S} \times \mathcal{A} \times \{2,\dots,K\}$ and $h=H$:
\begin{align*}
    (Q_H^k - Q_H^\star)(s,a) &= \sum_{i=1}^t \alpha_t^i \left[ (V_{H+1}^i - V_{H+1}^\star)(x_{s,a}^{H,i}) + \left[ \left(\mathbb{\hat{P}}_H^i - \mathbb{P}_H\right)V_{H+1}^\star \right](s,a) + b_i\right] \\
    &\geq \sum_{i=1}^t \alpha_t^i b_i \geq 0
\end{align*}
where we used the fact that the value functions at $H+1$ are zero. Then, assuming $(Q_{h+1}^k - Q_{h+1}^\star)(s,a) \geq 0$ for all $(s,a,k) \in \mathcal{S} \times \mathcal{A} \times \{2,\dots,K\}$, we have for any $h < H$:
\begin{align*}
    (Q_h^k - Q_h^\star)(s,a) &= \sum_{i=1}^t \alpha_t^i \left[ (V_{h+1}^i - V_{h+1}^\star)(x_{s,a}^{h,i}) + \left[ \left(\mathbb{\hat{P}}_h^i - \mathbb{P}_h\right)V_{h+1}^\star \right](s,a) + b_i\right] \\
    &\geq \sum_{i=1}^t \alpha_t^i (V_{h+1}^i - V_{h+1}^\star)(x_{s,a}^{h,i}) + \sum_{i=1}^t \alpha_t^i \Big( \left[ \left(\mathbb{\hat{P}}_h^i - \mathbb{P}_h\right)V_{h+1}^\star \right](s,a) + b_i \Big) \\
    &= \sum_{i=1}^t \alpha_t^i \Big( \min\left\{ H,\ \max Q_{h+1}^i(x_{s,a}^{h,i}, \cdot) \right\} - V_{h+1}^\star(x_{s,a}^{h,i}) \Big) + \sum_{i=1}^t \alpha_t^i \Big( \left[ \left(\mathbb{\hat{P}}_h^i - \mathbb{P}_h\right)V_{h+1}^\star \right](s,a) + b_i \Big) \numberthis \label{eq:lower} \\
    &\geq 0
\end{align*}
To see why the first term in Equation~\ref{eq:lower} is positive, let $y_{s,a}^{h,i} = \argmax_{a} Q_{h+1}^\star(x_{s,a}^{h,i}, a)$, then for all $i \in [K]$:
\begin{align*}
    V_{h+1}^\star(x_{s,a}^{h,i}) =  Q_{h+1}^\star(x_{s,a}^{h,i}, y_{s,a}^{h,i}) \stackrel{\text{\tiny\Circled{1}}}{\leq} Q_{h+1}^i(x_{s,a}^{h,i}, y_{s,a}^{h,i}) \leq \max Q_{h+1}^i(x_{s,a}^{h,i}, \cdot)
\end{align*}
where in {\tiny\CircledTop{1}}, we used the recursion assumption $(Q_{h+1}^i - Q_{h+1}^\star)(s,a) \geq 0$. The second term in Equation~\ref{eq:lower} is positive from the bonus design (cf. Equations~\ref{eq:b_t_lower} and \ref{eq:hat_p_true_p_bound_lower}).
\end{proof}

\subsection{Convergence of On-Policy Errors}

With Lemma~\ref{lem:alpha} and Lemma~\ref{lem:optimism} at hand, we are ready to establish the main lemma that shows Q-values convergence.

\begin{lemma}
Let $\left\{ P_h^t \right\}_{h=1}^{H}$ be a sequence that satisfy the following recursive relationship:
\begin{equation*}
    P_h^t = 4 \left(1+\frac{1}{H}\right) \left(c \sqrt{H^3\iota} + L \zeta \sqrt{t} \right) + \left(1+\frac{1}{H}\right)P_{h+1}^t
\end{equation*}
where $P_{H+1}^t \coloneqq 0$. Then, for any $p \in (0,1)$, letting letting $b_t = c\sqrt{\frac{H^3\iota}{t}} + L\zeta$, the following holds simultaneously for all $(s,h,k) \in \mathcal{S} \times [H] \times \{2,\dots,K\}$ with at least $1-p$ probability:
\begin{align*}
    0 \leq (V_h^k - V_h^\star)(s) \leq (Q_h^k - Q_h^\star)(s, \pi_h^k(s)) \leq \frac{P_h^t}{\sqrt{t}}
\end{align*}
where $t=k-1$.
\label{lem:convergence}
\end{lemma}
\begin{proof}
To start, we have:
\begin{align*}
    (V_h^k - V_h^\star)(s) &\stackrel{\text{\tiny\Circled{1}}}{\leq} (Q_h^k - Q_h^\star)(s, \pi_h^k(s))\\
    &\stackrel{\text{\tiny\Circled{2}}}{\leq} \sum_{i=1}^t \alpha_t^i (V_{h+1}^i - V_{h+1}^\star)(f(s,\pi_h^k(s)) + w_h^i + \varepsilon_{s,a}^{h,i}) + 2 \left( 1 + \frac{1}{H} \right) b_t \\
    &\leq \alpha_t^1 H + \sum_{i=2}^t \alpha_t^i (V_{h+1}^i - V_{h+1}^\star)(f(s,\pi_h^k(s)) + w_h^i + \varepsilon_{s,a}^{h,i}) + 2 \left( 1 + \frac{1}{H} \right) b_t \numberthis \label{eq:equation_1}
\end{align*}
{\tiny\CircledTop{1}} holds as $V_h^k(s) \leq \max Q_h^k(s, \cdot) = Q_h^k(s, \pi_h^k(s))$ and $V_h^\star(s) = \max_{a' \in \mathcal{A}} Q_h^\star(s, a') \geq Q_h^\star(s, \pi_h^k(s))$ while {\tiny\CircledTop{2}} holds from Lemma~\ref{lem:optimism}. Now, we need to argue that the error due to the initialization $\alpha_t^1 H$ decays fast enough. To that end, we can show from Equation~\ref{eq:alpha_i_t}:
\begin{align*}
    \alpha_t^1 = \alpha_1 \prod_{j=2}^t (1-\alpha_j) = \prod_{j=2}^t \frac{j-1}{H+j} \leq \prod_{j=2}^t \frac{j-1}{j} = \frac{1}{t}
\end{align*}
which implies that $\alpha_t^1 H \leq \frac{H}{t} \leq 2 \left(1+\frac{1}{H}\right) b_t$ for any nontrivial MDP. Thus, we have:
\begin{align*}
    (\ref{eq:equation_1}) \leq& \sum_{i=2}^t \alpha_t^i (V_{h+1}^i - V_{h+1}^\star)(f(s,\pi_h^k(s)) + w_h^i + \varepsilon_{s,a}^{h,i}) + 4 \left( 1 + \frac{1}{H} \right) b_t \\ 
    =& \sum_{i=2}^t \alpha_t^i (V_{h+1}^i - V_{h+1}^\star)(f(s,\pi_h^k(s)) + w_h^i + \varepsilon_{s,a}^{h,i}) + 4 \left( 1 + \frac{1}{H} \right) \left(c \sqrt{\frac{H^3\iota}{t}} + L\zeta \right) \numberthis \label{eq:equation_2}
\end{align*}
We complete the proof by induction. For the base case $h=H$, we have for any $(s,h,k) \in \mathcal{S} \times [H] \times \{2,\dots,K\}$:
\begin{align*}
    (V_H^k - V_H^\star)(s)
    \leq (Q_H^k - Q_H^\star)(s, \pi_H^k(s))
    \leq 4 \left(1+\frac{1}{H}\right) \left(c \sqrt{\frac{H^3\iota}{t}} + L\zeta \right) 
    = \frac{P_H^t}{\sqrt{t}}
\end{align*}
Then, assuming that $(V_{h+1}^k - V_{h+1}^\star)(s) \leq (Q_{h+1}^k - Q_{h+1}^\star)(s, \pi_{h+1}^k(s)) \leq \frac{P_{h+1}^t}{\sqrt{t}}$, we have for any $h < H$:
\begin{align*}
    (V_h^k - V_h^\star)(s) \leq (Q_h^k - Q_h^\star)(s, \pi_h^k(s)) &\leq \sum_{i=2}^t \alpha_t^i (V_{h+1}^i - V_{h+1}^\star)(f(s,\pi_h^k(s)) + w_h^i + \varepsilon_{s,a}^{h,i}) + 4 \left( 1 + \frac{1}{H} \right) \left(c \sqrt{\frac{H^3\iota}{t}} + L\zeta \right) \\
    &\stackrel{\text{\tiny\Circled{1}}}{\leq} \sum_{i=2}^t \frac{\alpha_t^i P_{h+1}^i}{\sqrt{i}} + 4 \left( 1 + \frac{1}{H} \right) \left(c \sqrt{\frac{H^3\iota}{t}} + L\zeta \right)  \\
    &\stackrel{\text{\tiny\Circled{2}}}{\leq} \sum_{i=1}^t \frac{\alpha_t^i P_{h+1}^i}{\sqrt{i}} + 4 \left( 1 + \frac{1}{H} \right) \left(c \sqrt{\frac{H^3\iota}{t}} + L\zeta \right) \\
    &\stackrel{\text{\tiny\Circled{3}}}{\leq} \left(1+\frac{1}{H}\right) \frac{P_{h+1}^t}{\sqrt{t}} + 4 \left( 1 + \frac{1}{H} \right) \left(c \sqrt{\frac{H^3\iota}{t}} + L\zeta \right) \\
    &\stackrel{\text{\tiny\Circled{4}}}{=} \frac{P_h^t}{\sqrt{t}}
\end{align*}
where {\tiny\CircledTop{1}} holds from the recursion assumption, {\tiny\CircledTop{2}} as $\alpha_t^1 P_{h+1}$ is positive, {\tiny\CircledTop{3}} from Lemma~\ref{lem:alpha_a}, and {\tiny\CircledTop{4}} by definition. This concludes the proof of the right-hand side of the lemma. The left-hand side follows directly from Lemma~\ref{lem:optimism} as $Q_h^k \geq Q_h^\star$, and thus $V_h^k \geq V_h^\star$.
\end{proof}
With simple recursion, one can show that the elements of the sequence $\left\{ P_h^t \right\}_{h=1}^{H}$ are bounded:
\begin{equation}
    P_h^t \leq 4(1+H-h) \left( c \sqrt{H^3\iota} + L \zeta \sqrt{t} \right) \left(1+\frac{1}{H}\right)^{1+H-h}
\label{eq:bound_P_h_t}
\end{equation}

\subsection{Proof of Theorems~\ref{thm:regret}, \ref{thm:bpi}, and \ref{thm:online_estimator}}

Finally, we are ready to bound the suboptimality gap $(V_h^\star - V_h^{\pi^k})$. For any $(s,h,k) \in \mathcal{S} \times [H] \times \{2,\dots,K\}$, we have with at least $1-p$ probability:
\begin{align*}
    \big( V_h^\star - V_h^{\pi^k} \big) (s) &\stackrel{\text{\tiny\Circled{1}}}{\leq} \big( V_h^k - V_h^{\pi^k} \big) (s)  \\
     &\stackrel{\text{\tiny\Circled{2}}}{\leq} (Q_h^k - Q_h^{\pi^k})(s, \pi_h^k(s))
    = (Q_h^k - Q_h^\star)(s, \pi_h^k(s)) + ( Q_h^\star - Q_h^{\pi^k})(s, \pi_h^k(s)) \\
    &\stackrel{\text{\tiny\Circled{3}}}{=} \big( Q_h^k - Q_h^\star \big)(s, \pi_h^k(s)) + \mathbb{P}_h \big( V_{h+1}^\star - V_{h+1}^{\pi^k} \big) (s, \pi_h^k(s)) \\
    &\stackrel{\text{\tiny\Circled{4}}}{\leq} \frac{P_h^{k-1}}{\sqrt{k-1}} + \mathbb{P}_h \big( V_{h+1}^\star - V_{h+1}^{\pi^k} \big) (s, \pi_h^k(s))
    \numberthis \label{eq:equation_3}
\end{align*}
where {\tiny\CircledTop{1}} holds as $V_h^k \geq V_h^\star$, {\tiny\CircledTop{2}} as $V_h^k(s) \leq \max Q_h^k(s, \cdot) = Q_h^k(s, \pi_h^k(s))$ and $V_h^{\pi^k}(s) = Q_h^{\pi^k}(s, \pi_h^k(s))$, {\tiny\CircledTop{3}} from Bellman Equations~\ref{eq:bellman}, and {\tiny\CircledTop{4}} from Lemma~\ref{lem:convergence}. Then, we continue by recursion. For the base case $h=H$, we have for any $(s,k) \in \mathcal{S} \times \{2,\dots,K\}$:
\begin{align*}
    \big( V_H^\star - V_H^{\pi^k} \big) (s) \leq \frac{P_h^{k-1}}{\sqrt{k-1}} = \sum_{h' = H}^{H} \frac{P_{h'}^{k-1}}{\sqrt{k-1}}
\end{align*}
Then, assuming $\big( V_{h+1}^\star - V_{h+1}^{\pi^k} \big) (s) \leq \sum_{h'=h+1}^{H} \frac{P_{h'}^{k-1}}{\sqrt{k-1}}$, we can show that:
\begin{align*}
    \big( V_h^\star - V_h^{\pi^k} \big) (s) &\leq  \frac{P_h^{k-1}}{\sqrt{k-1}} + ( V_{h+1}^\star - V_{h+1}^{\pi^k} \big)^\top \mathbb{P}_h(\cdot | s, \pi_h^k(s)) \\
    &\stackrel{\text{\tiny\Circled{1}}}{\leq} \frac{P_h^{k-1}}{\sqrt{k-1}} + \left( \sum_{h'=h+1}^{H} \frac{P_{h'}^{k-1}}{\sqrt{k-1}} \right) \mathds{1}^\top \mathbb{P}_h(\cdot | s, \pi_h^k(s)) \\
    &\stackrel{\text{\tiny\Circled{2}}}{=}  \frac{P_h^{k-1}}{\sqrt{k-1}} +  \sum_{h'=h+1}^{H} \frac{P_{h'}^{k-1}}{\sqrt{k-1}} = \sum_{h'=h}^{H} \frac{P_{h'}^{k-1}}{\sqrt{k-1}}
\end{align*}
where {\tiny\CircledTop{1}} follows as the recursion assumption holds for all $s \in \mathcal{S}$ and {\tiny\CircledTop{2}} holds as $\mathbb{P}_h(\cdot | s, \pi_h^k(s))$ is a probability vector. Thus, we have with probability at least $1-p$:
\begin{align*}
    \big( V_1^\star - V_1^{\pi^k} \big) (s) \leq \sum_{h=1}^{H} \frac{P_{h}^{k-1}}{\sqrt{k-1}} \stackrel{\text{\tiny\Circled{1}}}{\leq} \mathcal{O}\left( \sqrt{\frac{H^7\iota}{k-1}}+ L \zeta H^2 \right)
\end{align*}
where {\tiny\CircledTop{1}} follows from the upper bounds in Equation~\ref{eq:bound_P_h_t}. This establishes the claims in Theorem~\ref{thm:bpi}.

For the claims in Theorem~\ref{thm:regret}, we have:
\begin{align*}
    \sum_{k=1}^K \big( V_1^\star - V_1^{\pi^k} \big) (s_1^k) &\leq H + \sum_{k=2}^K \sum_{h=1}^{H} \frac{P_{h}^{k-1}}{\sqrt{k-1}} \\
    &\stackrel{\text{\tiny\Circled{1}}}{\leq} H + 8L\zeta H^2K + 8c\sqrt{H^7} \sum_{k=2}^K \sqrt{\frac{1}{k-1}} \\
    &\stackrel{\text{\tiny\Circled{2}}}{\leq} H + 8L\zeta H^2K + 8c\sqrt{H^7K} \sqrt{\sum_{k=1}^{K-1} \frac{1}{k}} \\
    &\stackrel{\text{\tiny\Circled{3}}}{\leq} H + 8L\zeta H^2K + 8c\sqrt{2\iota H^7 K\log{K}} \\
    &\leq \mathcal{O}\left( \sqrt{H^7K \iota \log{K}} + L\zeta H^2K \right)
\end{align*}
{\tiny\CircledTop{1}} holds from Equation~\ref{eq:bound_P_h_t}, {\tiny\CircledTop{2}} from Jensen's inequality, and {\tiny\CircledTop{3}} from the bound on the Harmonic series $\sum_{n=1}^N \frac{1}{n} \leq \log_2(N+1)$.

For Theorem~\ref{thm:online_estimator}, if we choose the bonuses such that:
\begin{align}
    b_t = c\sqrt{\frac{H^3\iota}{t}} + L\sqrt{\frac{d}{t}}
\end{align}
Then, we can use the same augments presented here without any further complications and the proof follows immediately.

\end{document}